%% file: main.tex
\documentclass{article}

\usepackage{microtype}
\usepackage{graphicx}
\usepackage{subfigure}
\usepackage{booktabs} 
\usepackage{listings}
\usepackage{lscape}
\usepackage{hyperref}

\usepackage[accepted]{icml2024}

\usepackage{amsmath}
\usepackage{amssymb}
\usepackage{mathtools}
\usepackage{amsthm}
\usepackage[nice]{nicefrac}
\input{mycommand}

\usepackage[capitalize,noabbrev]{cleveref}

\theoremstyle{plain}

\theoremstyle{definition}

\theoremstyle{remark}

\usepackage[textsize=tiny]{todonotes}

\icmltitlerunning{Generating CoTs with a pairwise-comparison approach to searching for the most promising intermediate thought}

\begin{document}

\twocolumn[
\icmltitle{Generating Chain-of-Thoughts with a Pairwise-Comparison Approach to\\ Searching for the Most Promising Intermediate Thought}



\icmlsetsymbol{equal}{*}

\begin{icmlauthorlist}
\icmlauthor{Zhen-Yu Zhang}{riken}
\icmlauthor{Siwei Han}{unc}
\icmlauthor{Huaxiu Yao}{unc}
\icmlauthor{Gang Niu}{riken}
\icmlauthor{Masashi Sugiyama}{riken,utokyo}
\end{icmlauthorlist}

\icmlaffiliation{riken}{Center for Advanced Intelligence Project, RIKEN}
\icmlaffiliation{unc}{University of North Carolina at Chapel Hill}
\icmlaffiliation{utokyo}{Graduate School of Frontier Sciences, The University of Tokyo}

\icmlcorrespondingauthor{Masashi Sugiyama}{sugi@k.u-tokyo.ac.jp}

\icmlkeywords{Machine Learning, ICML}

\vskip 0.3in
]



\printAffiliationsAndNotice{}  

\begin{abstract}
To improve the ability of the \emph{large language model}~(LLMs) to tackle complex reasoning problems, \emph{chain-of-thoughts} (CoT) methods were proposed to guide LLMs to reason step-by-step, enabling problem solving from simple to complex. State-of-the-art methods for generating such a chain involve interactive collaboration, where the learner generates candidate intermediate thoughts, evaluated by the LLM, guiding the generation of subsequent thoughts. However, a widespread yet understudied problem is that \emph{the evaluation from the LLM is typically noisy and unreliable}, potentially misleading the generation process in selecting promising intermediate thoughts. In this paper, motivated by Vapnik's principle, we use \emph{pairwise-comparison} evaluation instead of point-wise scoring to search for promising intermediate thoughts with the noisy feedback from the LLM. In each round, we randomly \emph{pair intermediate thoughts and directly prompt the LLM to select} the more promising one from each pair, allowing us to identify the most promising thoughts through an iterative process. To further alleviate the noise in the comparison, we incorporate techniques from ensemble learning and dueling bandits, proposing two variants of the algorithm. Experiments on three real-world tasks demonstrate the effectiveness of our proposed algorithm and verify the rationale of the pairwise comparison mechanism.
\end{abstract}

\input{./section/intro}
\input{./section/related}
\input{./section/method}
\input{./section/experiments}

\section{Conclusion}
\label{sec:conclusion}
This paper investigates a widespread but understudied problem of noisy feedback from LLMs in CoT generation tasks. Motivated by Vapnik's principle, we argue that for LLMs, the simultaneous comparison of two thoughts provides a more robust evaluation compared to individual value evaluations, and thus we propose a pairwise-comparison ToT approach C-ToT, approaching to searching for the most promising intermediate thought. The proposed method directly selects the most promising intermediate thought by pairwise comparison, and incorporates previous thoughts into the comparison to allow for rethinking. To further alleviate the noise in the comparison, we propose two variants of the C-ToT algorithm, and analyze the theoretical properties. Experiments on three real-world mathematical and reasoning tasks show the effectiveness of our proposed algorithm and verify the rationale of the pairwise comparison.

\section*{Acknowledgments}
HY was supported by Cisco Faculty Research Award. MS was supported by JST CREST Grant Number JPMJCR18A2. 

\section*{Impact Statement}
This research investigates a general problem of CoT generation with any LLM, where we take into account the noise in the feedback of the LLM. Therefore, when using LLMs for complex mathematical or logical reasoning problems, the user could benefit from our study from the aspect of generating a more effective CoT. The consequences of system failure and bias in the data are not applicable.


\bibliography{myRefs}
\bibliographystyle{icml2024}

\newpage
\appendix
\onecolumn

\input{./section/appendix}


\end{document}

%% file: mycommand.tex
\usepackage{amsmath}
\usepackage{amssymb}
\usepackage{mathtools}
\usepackage{amsthm}
\usepackage{dsfont}
\usepackage{lipsum}
\usepackage{bm,bbm}
\usepackage{bbding}
\usepackage{enumitem}
\usepackage{multirow}
\usepackage{graphicx}
\usepackage{subfigure}
\usepackage{adjustbox}
\usepackage[font=small,labelfont=bf]{caption}
\usepackage{subfloat}
\usepackage{float}
\usepackage{algorithm,algorithmic}
\usepackage{wrapfig}
\usepackage{booktabs}
\usepackage{verbatim}
\usepackage{setspace}
\usepackage{tcolorbox}

\def \x {\mathbf{x}}

\def \z {\mathbf{z}}

\def \O {\mathcal{O}}

\renewcommand{\algorithmicrequire}{ \textbf{Input:}}      
\renewcommand{\algorithmicensure}{ \textbf{Output:}}     

\renewcommand{\tilde}{\widetilde}
\renewcommand{\hat}{\widehat}

\usepackage{mathtools}
\usepackage{appendix}

\newtheorem{myLemma}{Lemma}

\newtheorem{myProp}{Proposition}

\theoremstyle{definition}

\newtheorem{myRemark}{Remark}

\renewcommand{\tilde}{\widetilde}
\renewcommand{\hat}{\widehat}

\def \x {\mathbf{x}}

\def \epsilon {\varepsilon}

\definecolor{DSgray}{cmyk}{0,1,0,0}

\usepackage[textsize=tiny]{todonotes}

\usepackage{prettyref}

\newcommand{\savehyperref}[2]{\texorpdfstring{\hyperref[#1]{#2}}{#2}}
\newrefformat{eq}{\savehyperref{#1}{Eq.~(\ref*{#1})}}
\newrefformat{eqn}{\savehyperref{#1}{Equation~\ref*{#1}}}
\newrefformat{lem}{\savehyperref{#1}{Lemma~\ref*{#1}}}
\newrefformat{lemma}{\savehyperref{#1}{Lemma~\ref*{#1}}}
\newrefformat{def}{\savehyperref{#1}{Definition~\ref*{#1}}}
\newrefformat{line}{\savehyperref{#1}{Line~\ref*{#1}}}
\newrefformat{thm}{\savehyperref{#1}{Theorem~\ref*{#1}}}
\newrefformat{corr}{\savehyperref{#1}{Corollary~\ref*{#1}}}
\newrefformat{cor}{\savehyperref{#1}{Corollary~\ref*{#1}}}
\newrefformat{sec}{\savehyperref{#1}{Section~\ref*{#1}}}
\newrefformat{app}{\savehyperref{#1}{Appendix~\ref*{#1}}}
\newrefformat{appendix}{\savehyperref{#1}{Appendix~\ref*{#1}}}
\newrefformat{assum}{\savehyperref{#1}{Assumption~\ref*{#1}}}
\newrefformat{ex}{\savehyperref{#1}{Example~\ref*{#1}}}
\newrefformat{fig}{\savehyperref{#1}{Figure~\ref*{#1}}}
\newrefformat{alg}{\savehyperref{#1}{Algorithm~\ref*{#1}}}
\newrefformat{rem}{\savehyperref{#1}{Remark~\ref*{#1}}}
\newrefformat{conj}{\savehyperref{#1}{Conjecture~\ref*{#1}}}
\newrefformat{prop}{\savehyperref{#1}{Proposition~\ref*{#1}}}
\newrefformat{proto}{\savehyperref{#1}{Protocol~\ref*{#1}}}
\newrefformat{prob}{\savehyperref{#1}{Problem~\ref*{#1}}}
\newrefformat{claim}{\savehyperref{#1}{Claim~\ref*{#1}}}
\newrefformat{que}{\savehyperref{#1}{Question~\ref*{#1}}}
\newrefformat{op}{\savehyperref{#1}{Open Problem~\ref*{#1}}}
\newrefformat{fn}{\savehyperref{#1}{Footnote~\ref*{#1}}}
\allowdisplaybreaks[3]

%% file: section/intro.tex

\section{Introduction}
\label{sec:intro}

\emph{Large language models} (LLMs), such as the GPT~\citep{brown2020language} and PaLM~\citep{chowdhery2023palm}, have recently demonstrated remarkable capabilities in a variety of real-world tasks. However, current LLMs still face limitations when dealing with complex tasks, especially those involving multi-step reasoning, such as mathematical or reasoning problems~\citep{rae2021scaling,NIPS2022:CoT}. To deal with such implicit complexity, \emph{chain-of-thoughts} (CoT) approaches were proposed~\citep{NIPS2022:CoT,ICLR2022:CoTensemble,NIPS2023:ToT}. These approaches were proposed to use an incorporation of intermediate steps of reasoning (intermediate ``thought''), enabling the LLM to reason progressively, first generating intermediate solutions for simpler problems to incrementally improve its capacity to handle complicated tasks. Therefore, the key challenge is to design an effective CoT generation algorithm that guides the LLM towards desired solutions through step-by-step reasoning.

There is a fruitful line of work that considers the CoT generation problem. The pioneering work uses manual design prompts to let the LLM generate a CoT by itself~\citep{NIPS2022:CoT,ICLR2022:CoTensemble}. This line of research was recently extended by the \emph{score-based tree-of-thoughts} (S-ToT) approaches~\citep{NIPS2023:ToT,long2023:ToT_long}, where the CoT generation is framed as an \emph{interactive process} with the algorithm and the LLM. These approaches generate a set of candidate intermediate thoughts each round and ask the LLM to score them and select the most promising ones. The next thoughts are then generated based on these selected ones, creating a tree-like data structure. A search algorithm, such as deep-first search, is used to identify the most promising CoT in the tree (see the detailed illustration in Figure~\ref{fig:illustration}). 

While these methods have shown remarkable empirical success, they rely on an accurate score evaluation of each intermediate thought by the LLM. However, it is important to notice that: \emph{LLM scores are often noisy}. For example, the LLM may give different responses to different prompts, even though these prompts convey the same meaning~\citep{lu2022fantastically}. The noisy nature of LLM feedback introduces new problems in the selection of the most promising intermediate thoughts and the subsequent generation of the tree structure. Therefore, it is crucial to make the CoT generation algorithms robust to the noisy feedback from LLMs.

Several preliminary approaches have been proposed to mitigate such noise in the LLM feedback, including estimating uncertainty from the semantic aspect~\citep{kuhn2022semantic} or ensembling multiple thoughts~\citep{ICLR2022:CoTensemble}. However, getting an accurate point-wise estimate for each intermediate thought could be resource-intensive, requiring the construction of an additional model~\citep{paul2023refiner} or multiple queries~\citep{ICLR2022:CoTensemble}: see also Figure~\ref{fig:cost} and Table~\ref{tab:cost-QA} to~\ref{tab:cost-sudoku} in our experiments. Fortunately, in the context of CoT generation, our focus is on identifying the most promising chain. 
Motivated by Vapnik's principle~\citep{vapnik1991principles}, we do not need to solve a more general and difficult problem as an intermediate step, i.e. to estimate an accurate point-wise score for each intermediate thought. Instead, we can focus directly on identifying the most promising one in each round. However, it is still impractical to directly vote on all intermediate thoughts to identify the most promising one by LLMs due to the input length limit and the ``lost in the middle'' phenomenon~\citep{liu2024lost}. 

\emph{We argue that for LLMs, comparing two thoughts simultaneously provides a more robust evaluation than assigning individual scores}. We aim to leverage the comparison of two thoughts instead of evaluating a single thought in isolation, thereby providing a feasible alternative for identifying the most promising intermediate thought.
This argument is well established in human cognition, as seen in mathematical problems, where it is often more feasible to approximate which thought is better by comparison than by considering and evaluating them separately. We also observe similar phenomena that LLMs to generate a more reliable evaluation in the experiments on the Sudoku task, as shown in Figure~\ref{fig:examples}, where the LLM successfully identifies the better option given two intermediate solutions, but struggles to assign the correct value to intermediate thoughts individually.

Based on the above insights, we propose a pairwise comparison-based algorithm for CoT generation to alleviate the noise in the LLM feedback and to find the most promising intermediate thoughts each round. In each round, we randomly pair all the intermediate thoughts and directly ask the LLM to compare and select the more promising one from each pair, keeping the selected one and discarding the other. Then we repeat this procedure so that we get a small set of most promising intermediate thoughts, and subsequently, we generate the next thoughts based on these selected ones. This mechanism allows us to use a direct pairwise comparison to identify the promising thoughts with a more robust evaluation. We also propose to include previous thoughts in the tree structure for comparison to mitigate the noisy nature of LLM's feedback. Taking these two points into account, we frame the problem as an iterative process and propose a general CoT generation algorithm called \emph{comparison-based tree-of-thoughts} (C-ToT). To further model the noise in the comparison, we resort to the techniques of ensemble and best-arm identification with dueling feedback~\citep{ICML2017:Knockout} and propose two variants of the proposed C-ToT algorithm. Through experiments on three real-world reasoning problems, we demonstrate the effectiveness of our proposed approaches and verify the rationale of the pairwise comparison mechanism. Our main contributions are:
\begin{itemize}%
    \vspace{-4mm}
    \item[(1)] We investigate the problem of noisy feedback in the CoT generation, which is widespread but understudied.
    \vspace{-6mm}
    \item[(2)] Motivated by Vapnik's principle, we propose a pairwise-comparison based approach for CoT generation that exploits noisy feedback from the LLMs.
    \vspace{-2mm}
    \item[(3)] We proposed two variants of C-ToT that further account for different types of noise in the comparison feedback.
    \vspace{-4mm}
\end{itemize}

\begin{figure}
\begin{minipage}[b]{.99\linewidth}
\centering
\includegraphics[clip, trim=1.31cm 7.86cm 6.79cm 0.85cm, width=0.9\textwidth]{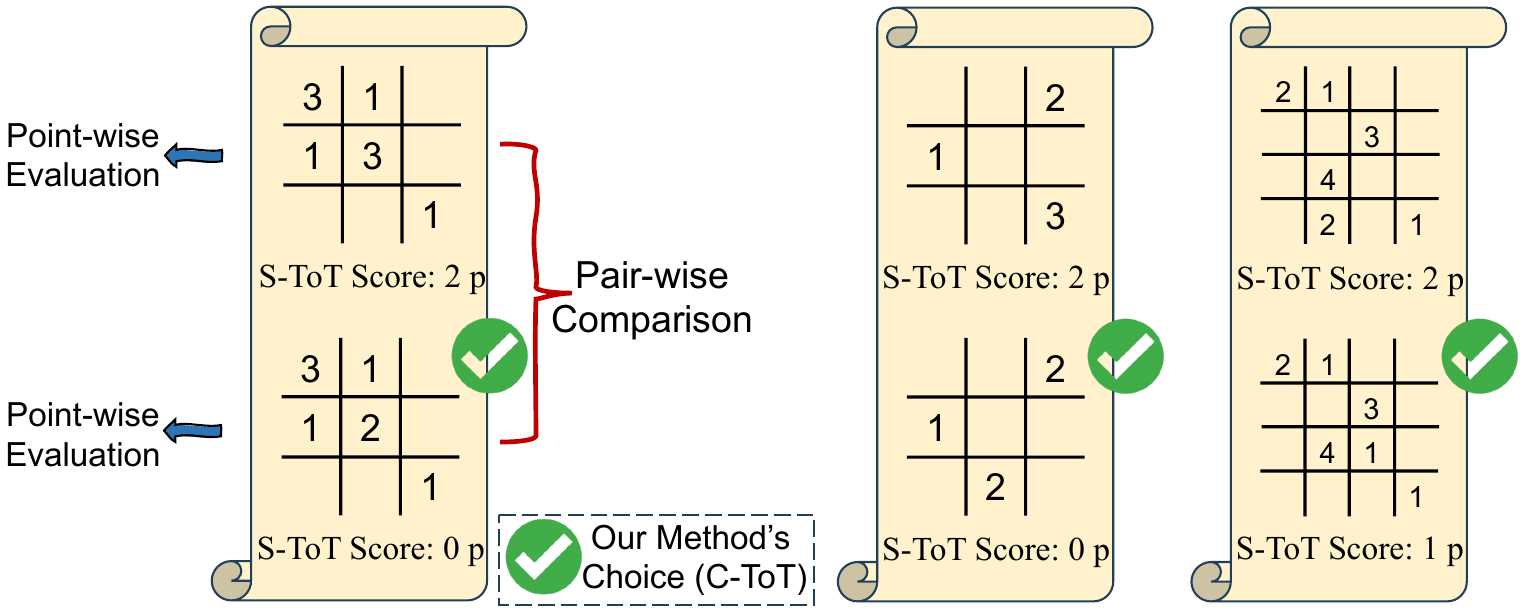}
\caption{A demonstration of point-wise evaluation vs. pair-wise comparison based on real experimental results in Sudoku puzzles. The point-wise evaluation algorithm (S-ToT) assigns hard scores to each intermediate thought (the higher, the better), while our proposed algorithm (C-ToT) uses pair-wise comparison to obtain the more promising thoughts (green tick). In these cases, the LLM assigns \textbf{incorrect scores}, but it makes a \textbf{correct comparison}.}
\label{fig:examples}
\end{minipage}
\vspace{-4mm}
\end{figure}

%% file: section/related.tex

\section{Related Work}
\label{sec:related} 

\textbf{CoT Generation.} Generating appropriate CoT for LLMs to enhance their inference power is a critical problem in real-world applications. Previous work has explored task-specific training algorithms for identifying the CoT, including creating semantic graphs~\citep{xu2021exploiting}, refining the model through human-annotated CoT~\citep{cobbe2021training}, or learning an additional extractor using heuristic-driven pseudo CoT~\citep{chen2019multi}. Different from these approaches, the LLM-based CoT generation is used directly during inference, coupling the generation process with an LLM. In these approaches, the LLM guides the CoT generation, eliminating the need for additional training.

The pioneering work in LLM-based CoT generation introduces intermediate thoughts sequentially between the input query and LLM's response. By simply prompting the LLM to ``think step by step'', this strategy has been shown to significantly improve several tasks over directly asking the LLM the original question, such as mathematical puzzles~\citep{NIPS2022:CoT} or other general mathematical reasoning problems~\citep{drori2022neural}. Due to the noisy nature of the LLM feedback, robustness can be improved by using an ensemble of different CoTs~\citep{ICLR2022:CoTensemble}.

To further improve the effectiveness of CoT generation, the score-based tree-of-thoughts generation algorithm was introduced independently by~\citet{NIPS2023:ToT} and~\citet{long2023:ToT_long}. They model the CoT generation process as a tree generation and search process. A single node in the tree represents an intermediate thought. Starting from a given node, the thought generator constructs a set of new nodes and the LLM generates scores for each node as an evaluation. Finally, the timing of the tree expansion is determined by the search algorithm used (e.g., breadth-first or depth-first search). In addition, this search algorithm can also provide capabilities including backtracking from unpromising thoughts. Further research extended the tree structure to a graph, such as the graph-of-thoughts~\citep{besta2023:GoT}, allowing the distillation of knowledge about entire network of thoughts. However, these methods cannot handle the noisy evaluation feedback caused by the LLM itself.

\textbf{Self-Reflection.}
Rather than interacting with LLMs to generate a step-by-step reasoning chain, self-reflection approaches involve LLMs directly offering an initial thought chain to the query, followed by iterative refinement of the whole chain. \citet{madaan2023self} and~\citet{paul2023refiner} introduced the ``self-reflection'' mechanism, using the LLMs to provide feedback to their generation candidates and then fine-tuning. \citet{paul2023refiner} updated the model to explicitly generate intermediate thoughts while interacting with a critic model that provides automated feedback on the reasoning. These methods introduced new models to provide evaluation for the intermediate thoughts, but these critical models still do not always provide perfect evaluation. Furthermore, for complex problems that require sequential reasoning, such as the Game of 24, where the next thought should be generated and evaluated based on previous ones, the C-ToT generation could be more appropriate.

\textbf{Uncertainty Quantification in LLMs.}
This is a recent interest that aims to evaluate the confidence of a given answer by the LLM itself. Some work considered letting the LLM provide the confidence~\citep{JAIR2021:cleanlab,kadavath2022language} by retraining the model. Another line of work considered designing entropy-based measures~\citep{kuhn2022semantic}, or generating multiple outputs to obtain an uncertainty measure~\citep{ICLR2022:CoTensemble}. Although they can be included in the CoT generation, they introduce a high computational cost during testing, particularly when obtaining an accurate score for each intermediate thought.

%% file: section/method.tex

\begin{figure*}
\begin{minipage}[b]{.99\linewidth}
\centering
\includegraphics[clip, trim=0.00cm 5.60cm 0.00cm 1.11cm, width=0.98\textwidth]{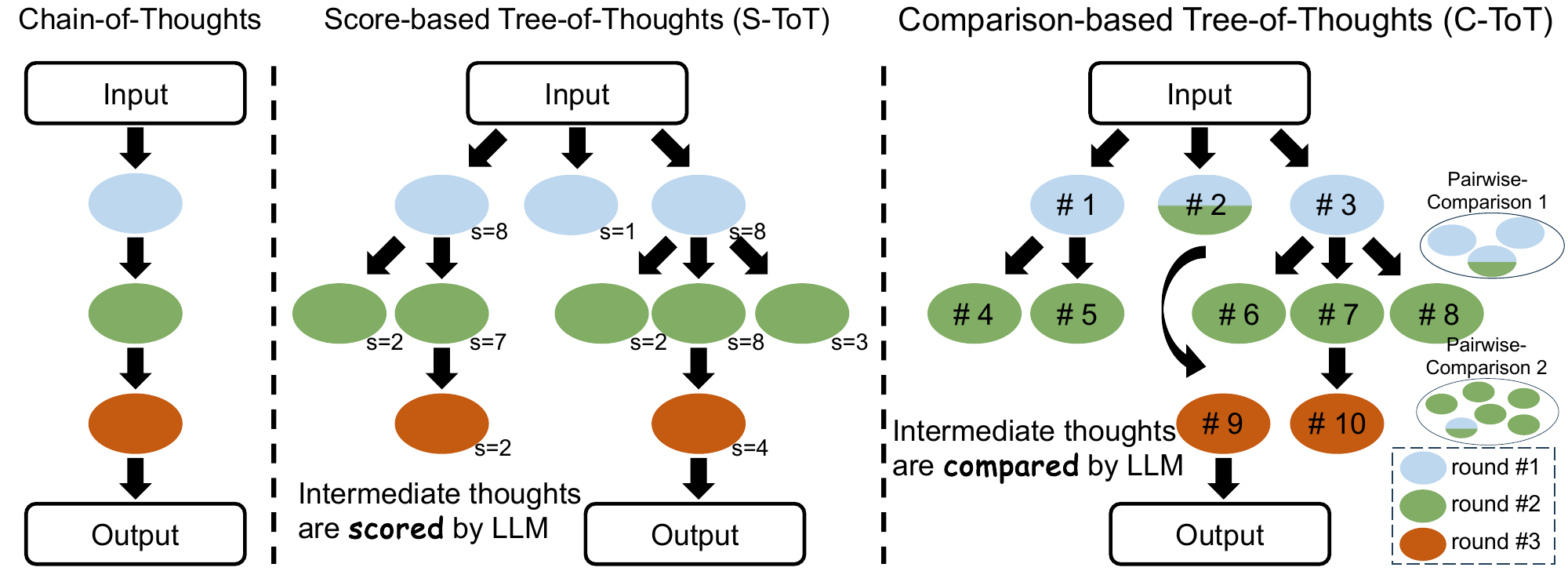}
\caption{Schematic illustration of previous CoT and S-ToT approaches with our proposed C-ToT approach for CoT generation with LLMs. Each circle box represents an intermediate thought, which is a coherent sequence of language or equations that serves as an intermediate step in problem solving. In the S-ToT method, \textbf{each intermediate thought is scored} by the LLM (denoted by $s$ in the figure), and the searching algorithm considers the highest-scoring ones as the most promising and then generates next intermediate thoughts based on them. In the C-ToT approach, we use \textbf{pairwise comparison} with the LLM in each round to find the most promising intermediate thoughts and then generate the next thoughts. Meanwhile, we include all previous intermediate thoughts in the comparison.}
\label{fig:illustration}
\end{minipage}
\vspace{-4mm}
\end{figure*}

\section{Our Approach}
\label{sec:method}
In this section, we first introduce the proposed comparison-based ToT generation algorithm, which is a general framework that generates CoT with a pairwise comparison mechanism to find the most promising intermediate thought. To further alleviate the noise in LLM's comparison feedback, we propose two different instantiations of our framework with theoretical analysis.

\subsection{CoT Generation via Pair-wise Comparison}
We first introduce the comparison-based ToT framework, where the key mechanism is the selection of the most promising thoughts among all candidates in each round.

We illustrate our proposed algorithm and compare it with previous approaches in Figure~\ref{fig:illustration}. The CoT approaches directly ask the LLM to generate a CoT. The S-ToT approaches ask the LLM to score each intermediate thought and select the highest-scoring ones to generate the next layer. Different from these methods, we propose a pairwise-comparison approach to searching for the most promising intermediate thoughts. Note that with LLMs, due to feedback noise and input limitations, we cannot do a listwise voting that directly asks the LLM to sort all the intermediate thoughts. Let $Z$ be the set of all candidate intermediate thoughts, and we want to select the most promising $K$ thoughts from it. The comparison iterates as follows: we randomly pair thoughts from the set and select only the winner in each pair, thereby halving the size of $Z_i$ to $|Z_i|/2$, where $Z_i$ denotes the set in the $i$-th iteration. After at most $K\times\log_2 |Z|$ rounds, we can identify the $K$ most promising intermediate thoughts by such direct comparison. In practice, we can do one iteration of comparison and keep the remaining $K$ thoughts in the last few rounds. For each pair, we compare thoughts $a$ and $b$ with the LLM by asking which one is better, using different prompts with $n$ times, where $n\geq 1$. We defer the implementation details to Section~\ref{sec:experiment-setup}.

We take previous unselected intermediate thoughts into comparison to explore possibly valuable but mis-evaluated thoughts caused by the noise in feedback. This is because the evaluation of intermediate thoughts may not always be accurate, and the generation of the tree structure may miss valuable intermediate thoughts in previous iterations. In the seminal research of S-ToT~\citep{NIPS2023:ToT,long2023:ToT_long,besta2023:GoT}, the thought generator uses a backtracking mechanism to revisit previous thoughts when the current ones fall below a certain threshold. While this strategy aims to rescue promising thoughts, its efficiency is questionable because it may delay the exploration of previously valuable but incorrectly scored thoughts. In addition, backtracking only occurs after a thought has fallen below a manually chosen threshold, which is hard to know in advance.

Motivated by these shortcomings and the efficiency of our pairwise-comparison mechanism, we maintain a repository of previous intermediate thoughts. At each round, we include previously unselected thoughts in the comparisons, rather than relying on a fixed threshold. As illustrated in Figure~\ref{fig:illustration}, during the pairwise comparison in the second layer, we include the intermediate thought that was not selected in the first layer. This mechanism ensures that the algorithm has the flexibility to revisit previous thoughts based on the comparison results from the LLM in each round.

We formulate the comparison-based ToT generation as an iterative interaction between the thought generation and the LLM. Take the C-ToT illustration in Figure~\ref{fig:illustration} as an example. The algorithm starts by generating intermediate thoughts \#1 to \#3 based on the input. Following a pairwise comparison mechanism, thoughts \#1 and \#3 are selected, leading to the generation of new intermediate thoughts, namely \#4 to \#8. In the second layer, thought \#2, thoughts \#4 and \#5 (linked to \#1), along with thoughts \#6 to \#8 (linked to \#3) are compared, subsequently resulting in the selection of thought \#2 and thought \#7 (linked to \#3), and the generation of the next intermediate thoughts.

Formally, we denote an intermediate thought by $\z$. In round $t$, $Z^t$ represents the set of candidate intermediate thoughts for comparison, and $\hat{Z}^t$ denotes the selected intermediate thoughts. In a sequence of $T$ rounds, in the first round the algorithm generates a set of thoughts $Z^1 = \{\z^1_i\}_{i=1}^{m}$ based on the query, where $m$ denotes the set size. Then, the comparison-based ToT selects $K$ most promising intermediate thoughts based on the comparison result from the LLM. We denote the selected set of thoughts by $\hat{Z}^1 = \{\z^1_j\}_{j\in[K]}$. In the second round, the algorithm generates the new intermediate thoughts based on each selected thought\footnote{Each newly generated intermediate thought will contain the information about all its parents}. After $T$ rounds, we can get $K$ most promising thoughts, and all of them contain information about their parent nodes, thus formulating as $K$ CoTs. Therefore, this iterative process facilitates the refinement and selection of thoughts over multiple rounds. For each pair of thoughts, we use a direct comparison to identify the more promising one. We call such a direct comparison method ``Standard Mode''. We summarize the proposed approach in Algorithm~\ref{alg:DoT}.

\begin{figure}[!t]
\vspace{-2mm}
\begin{minipage}{0.48\textwidth}
\begin{algorithm}[H]
\caption{C-ToT Algorithm}
\label{alg:DoT}
\begin{algorithmic}[1]
\STATE \textbf{Input:} Query $\x$, comparison times $n$, number of intermediate thoughts generation $m$, number of selected thoughts $K$, depth of the tree $T$.
\STATE Generate initial thoughts $Z^1$ of size $m$ with query $\x$
\FOR{$t=1$ {\bfseries to} $T$}
\IF{Standard Mode}
\STATE Pair thoughts in $Z^t$ randomly
\WHILE{$|Z^t|>K$}
\FOR{every pair $(a,b)$ in $Z^t$}
\STATE Compare thoughts $a$ and $b$ by LLM $n$ times, then take a majority vote. If $a$ wins, keep thought $a$ in $Z^t$ and drop $b$, and vice versa.
\ENDFOR
\ENDWHILE 
\STATE Denote $\hat{Z}^t$ by the remaining thoughts
\ELSE
\STATE Call Algorithm~\ref{alg:Knockout} to obtain $\hat{Z}^t$
\ENDIF
\STATE Generate the next $m$ thoughts for each thought in $\hat{Z}^t$.
\ENDFOR
\end{algorithmic}
\end{algorithm}
\end{minipage}
\vspace{-4mm}
\end{figure}

\begin{myRemark}[Comparison Complexity]
In our approach, we keep all previous intermediate thoughts to compare in each round. This may affect the operational efficiency and exceed the storage limit. To improve the efficiency, we can introduce a counter for each intermediate thought to track its comparison frequency. If the comparison count of an intermediate thought exceeds a threshold, we can remove it from the tree. Since the comparison in each round is independent of each other, we could use parallel computing to improve the efficiency of the algorithm, or exploit more efficient machine learning techniques to schedule computational resources more adaptively and efficiently~\citep{zhou2023theoretical} in the future. If the tree depth is $T$, the total number of comparisons required is less than the order of $\O(nTK\log(m))$. 
\end{myRemark}

\begin{myRemark}[Token Cost]
The token costs of our proposed C-ToT approach and the S-ToT approaches are task-specific and generally incomparable. The C-ToT approach could discover valuable but misevaluated previous intermediate thoughts earlier than the S-ToT method. However, it may introduce more token overhead as we compare these thoughts multiple times. Therefore, we are better suited to the problem where the initial intermediate thoughts are more uncertain. We provide the token cost analysis in \textbf{Appendix~\ref{app:experiment-cost}}.
\end{myRemark}

\subsection{Instantiations and Analysis}

In the proposed C-ToT framework, we use a direct comparison for each pair of thoughts. Although our C-ToT method explores two sample information compared to the S-ToT methods that use only single sample information, the comparison feedback could still be inaccurate.  We offer two methods to select the winning thought in a pair.

\textbf{Standard.} Suppose the comparison difficulty of each pair is the same. Inspired by the ensemble algorithms in CoT generation~\citep{ICLR2022:CoTensemble}, we can improve the robustness of the comparison feedback by setting $n>1$ in the ``Standard Mode'', so that we compare the two thoughts in each pair for $n$ times and take majority voting output.

\textbf{Dueling.} We consider a more general assumption of noisy comparisons, where we only assume an unknown ranking of the $M_t$ thoughts at round $t$. This implies that the comparison difficulty of each pair varies, requiring a different number of comparisons for each pair. If two thoughts $a$ and $b$ are compared, thought $a$ is chosen with some unknown probability $p(a, b)$ and $b$ is chosen with $p(b, a) = 1-p(a, b)$, where the higher-ranked one has probability $\geq 1/2$. Repeated comparisons are independent of each other.

We formulate it as a best-arm identification problem with dueling feedback~\citep{yue2012k,ICML2017:Knockout}, propose a dueling bandits instantiation of the C-ToT framework, and analyze its properties. For each pair, we keep the empirical probability $\hat{p}_a$, a proxy for $p(a, b)$. We also maintain a confidence value $\hat{c}$ s.t., w.h.p., $\hat{p}_a \in (p(a,b) - \hat{c}, p(a,b) + \hat{c})$. We stop the comparisons when it is sure of the winner or when it reaches its comparison budget $n$. If it reaches $n$ comparisons, it outputs the element with more wins, randomly breaking ties. During comparison, we also compare two elements $a$, $b$ with LLM by query them with different prompts. We summarize the proposed instantiation in Algorithm~\ref{alg:Knockout}. Here stochasticity $\gamma$ models the problem hardness. 

\begin{figure}[!t]
\vspace{-2mm}
\begin{minipage}{0.48\textwidth}
\begin{algorithm}[H]
\caption{Knockout}
\label{alg:Knockout}
\begin{algorithmic}[1]
\STATE \textbf{Input:} Set $Z$, bias $\epsilon$, confidence $\delta$, stochasticity $\gamma$, $i=1$
\WHILE{$|Z|>K$}
\STATE Pair thoughts in $Z$ randomly
\FOR{every pair $(a,b)$}
\STATE Set bias $\epsilon=\frac{(2^{1/3}-1)\epsilon}{\gamma 2^{i/3}}$, confidence $\delta=\frac{\delta}{2^i}$, $\hat{p}_a=1/2$, $\hat{c}=1/2$, $n=\frac{1}{2\epsilon^2}\log\frac{2}{\epsilon}$, $r=0$, $w_a=0$
\WHILE{$|\hat{p}_a-1/2| \leq \hat{c}-\epsilon$ and $r \leq n$}
\STATE Compare thoughts $a$ and $b$ by LLM. if thought $a$ wins, $w_a = w_a+1$, and vice versa.
\STATE $r=r+1$, $\hat{p}_a=\frac{w_a}{r}$, $\hat{c}=\sqrt{\frac{1}{2r}\log\frac{4r^2}{\delta}}$
\IF{$\hat{p}_a \leq 1/2$}
\STATE Keep thought $b$ in $Z$ and drop $a$, break.
\ELSE
\STATE Keep thought $a$ in $Z$ and drop $b$, break.
\ENDIF
\ENDWHILE
\ENDFOR
\STATE $i=i+1$
\ENDWHILE
\STATE Return $\hat{Z}$ by the remaining thoughts
\end{algorithmic}
\end{algorithm}
\end{minipage}
\vspace{-4mm}
\end{figure}

\textbf{Analysis.}
First, we introduce some definitions. Given a set of thoughts $Z=\{\z_1, ..., \z_M\}$ of size $M$. Suppose there is an unknown underlying ranking function $r: \mathcal{Z} \mapsto \mathbb{N}$ that ranks all the thoughts. Let $r(\z_1), ..., r(\z_M)$ be the ranking of the thoughts, such that when two elements $\z_a$ and $\z_b$ are compared, the higher ranked one is selected first, e.g. $r(\z_a) < r(\z_b)$. We define the $\epsilon$-maximum via the $(\epsilon, \delta)$-PAC paradigm, which requires that the output is likely to be close to the intended value. Specifically, given $\epsilon>0, \delta>0$, with probability $\geq 1-\delta$, the maximum selection must produce an element $a$ such that for $b$, with $r(b) = M$, $p(a,b)\geq \frac{1}{2}-\epsilon$. We call such an output $\epsilon$-maximum. 

\begin{myLemma}[Theorem 3 in~\cite{ICML2017:Knockout}]
\label{Lemma:Knockout}
Knockout$(Z, \epsilon, \delta)$ uses $\O(\frac{\gamma^2|Z|}{\epsilon^2}\log\frac{1}{\delta})$ comparisons and with probability at least $1-\delta$, outputs an $\epsilon$-maximum.
\end{myLemma}

\begin{myProp}
\label{Prop:CToT}
Suppose that the depth of the tree is $T$, and thoughts in the shallower layers are more promising than those in the deeper ones. Then, the probability of missing the $\epsilon$-maximum promising thoughts in the $\tau$-th layer is $1-\delta^{\tau}$ with at most $\O(\frac{\gamma^2\sum_{i=1}^{T}|Z^i|}{\epsilon^2}\log\frac{1}{\delta})$ comparisons required for generating the whole tree of thoughts.
\end{myProp}

\begin{myRemark}
Proposition~\ref{Prop:CToT} is directly derived from Lemma~\ref{Lemma:Knockout} by the union bound. Proposition~\ref{Prop:CToT} shows that, under the general assumption of noisy comparisons and utilizing our proposed pairwise-comparison approach, valuable intermediate thoughts will still not be overlooked, especially for the thoughts in the shallow layer, which may be more uncertain as they appear at the beginning of the ToT generation. We leave the detailed proofs to \textbf{Appendix~\ref{app:proofs}}.
\end{myRemark}

%% file: section/experiments.tex

\section{Experiments}
\label{sec:experiments}

We test our proposed algorithm in three real-world tasks: \emph{question answering}~(QA), as well as mathematical reasoning tasks, namely, the Game of 24 and Sudoku Puzzles. The LLM employed in experiments is GPT-3.5-turbo-1106.

\subsection{Experiment Setup}
\label{sec:experiment-setup}

\textbf{Contenders Setup.}
Our evaluation firstly includes a comparison with a baseline method that directly queries the LLM for the final result (we denote it as Direct); three state-of-the-art contenders: CoT~\citep{NIPS2022:CoT}, SC-CoT~\citep{ICLR2022:CoTensemble}, and SToT~\citep{NIPS2023:ToT}. For the CoT method, we query the LLM directly to get the final answer, following the settings as in~\citep{NIPS2022:CoT}. For the SC-CoT method, for a fair comparison, if without further notice, we set the CoT number approximately the same number of tokens with our proposed algorithm. Specifically, 15 samples were generated for each question, using the same settings as in~\citep{ICLR2022:CoTensemble}, with the final answer determined by majority voting. The SToT approach is also implemented identically to the setting in~\citet{NIPS2023:ToT}. For the depth of the tree in the SToT algorithm, we set it equal to the depth with our proposed algorithms C-ToT (Stand.) and C-ToT (Duel.). 

We further propose three contenders for comparison to test the effectiveness of our proposed algorithm. A robust implementation of SToT is proposed that follows the main idea of SC-CoT to account for noise in the LLM feedback, called SC-SToT. Two variants of SToT that equipped with our proposed mechanisms are also included, denoted as Comp-SToT, Back-SToT. Specifically,
\begin{itemize}%
    \vspace{-2mm}
    \item[(1)] SC-SToT: We denote the self-consistent variant of the ToT algorithm as SC-SToT, short for Self-Consistent ToT algorithm. To alleviate feedback noise in the ToT algorithm, we draw inspiration from the self-consistent CoT generation algorithm~\citep{ICLR2022:CoTensemble}. Our proposal involves querying the LLM multiple times during intermediate thought evaluation in the ToT algorithm and using the majority voting results as the final evaluation in the SC-SToT. This contender is a direct extension of the ToT algorithm to account for the feedback noise of the LLM, but the cost is very high as it scores each intermediate thought multiple times;
    \vspace{-2mm}
    \item[(2)] Comp-SToT: We replace the score-based evaluation in the original SToT algorithm with our proposed pairwise comparison approach and refer to this variant algorithm as Comp-SToT.
    \vspace{-2mm}
    \item[(3)] Back-SToT: We replace the search algorithm in the original SToT algorithm with our proposed mechanism, which retains all previous intermediate thoughts with their corresponding socres and takes the highest scoring thoughts as the most promising ones. We call this variant algorithm as Back-SToT.
    \vspace{-2mm}
\end{itemize}

In addition, we include three state-of-the-art algorithms into comparison: PoT~\citep{chen2023program}, Self-Refine~\citep{madaan2023self} and GoT~\citep{besta2023:GoT}. Detailed experimental results can be found in \textbf{Appendix~\ref{app:experiments-summary}}.

For our proposed C-ToT approaches, we denote the C-ToT algorithm in ``Standard Mode'' by C-ToT (Stand.) and set the number of comparisons $n$ to 1. We denote the C-ToT algorithm that considers the general comparison noise by C-ToT (Duel.), and set the maximum number of comparisons to 3 and set $\gamma =0.1$. All experiments are repeated 3 times.

\textbf{Intermediate Thoughts Generation.} In general, different tasks should have different thought generators. Exploiting problem properties is essential to effectively design the intermediate thoughts. We follow the setting of~\cite{NIPS2023:ToT} to generate the intermediate thoughts. For example, we generate the thoughts as a few words, as in QA; as a line of equations, as in the Game of 24; or as an intermediate solution in the Sudoku puzzle. We defer the implementation for prompts and the cost comparison of our approaches and other contenders to \textbf{Appendix~\ref{app:experiment-implementation}} and \textbf{Appendix~\ref{app:experiment-cost}}.

\subsection{Question Answering}
\textbf{Task Setup.} 
We first test the performance of our proposed algorithm on the question answering tasks using the AQuA dataset~\citep{ling2017program}, which comprises 254 arithmetic reasoning tasks aimed at assessing logical abilities through various mathematical computation problems. Each question in this dataset is accompanied by five multiple-choice options, labeled from A to E. We follow the experimental protocol as it is in the work of~\citet{ICLR2022:CoTensemble}. The accuracy of the responses is gauged by comparing the generated answers with the standard solutions. Results on other QA datasets can be found in \textbf{Appendix~\ref{app:experiments-summary}}.

\textbf{C-ToT Setup.}
For the AQuA dataset, we set the maximum depth of tree of thoughts to 3. For each intermediate thought selected, we set $m=12$, thus generating 12 new intermediate thoughts as the next step. The maximum number of selected thoughts $K$ per layer is set to 3. Therefore, starting from the ``question'' as the root, all newly generated thoughts are compared, and we select 3 most promising intermediate thoughts to generate the next step.

In the question answering task, it is difficult to set a fixed length for the C-ToT, i.e., an intermediate thought may already summarize the answer before reaching the maximum length of the C-ToT. For those selected intermediate thoughts that have already reached an answer, we add them to the ``answer list'' and do not include them in the comparison in the next round. For those intermediate thoughts that have already reached an answer but were not selected, we will include them in the comparison in the next round. This mechanism thus gives excluded answers a chance to be included in the ``answer list'' by subsequent comparisons. After $T$ rounds of thought generation and comparison for selection, the selected chains are appended to the ``answer list'', and a majority voting mechanism is used on the ``answer list'' to determine the final answer. We leave the implementation details of the thought generator and the comparison prompt for the question answering task to \textbf{Appendix~\ref{app:experiment-implementation-QA}}.

\begin{figure}
    \centering
\begin{minipage}[b]{.34\linewidth}
\centering
\resizebox{.99\textwidth}{!}{
\begin{tabular}{lc}
\hline
Method & Accuracy \\
\hline
Direct & 24.8\% \\
CoT & 42.3\% \\
SC-CoT & 58.4\% \\
SToT & 57.1\% \\
\hline
SC-SToT & 57.6\% \\
Comp-SToT & 59.0\% \\
Back-SToT & 58.0\% \\
C-ToT (Stand.) & \textbf{61.4}\% \\
C-ToT (Duel.) & \textbf{63.0}\% \\
\hline
\end{tabular}
}
\captionof{table}{Average accuracy on AQuA.}
\label{tab:QA}
\end{minipage}%
\hspace{8mm}
\begin{minipage}[b]{.46\linewidth}
\centering
\resizebox{.99\textwidth}{!}{
\includegraphics[clip, trim=0.25 0.09cm 0.28cm 0.54cm, width=0.99\textwidth]{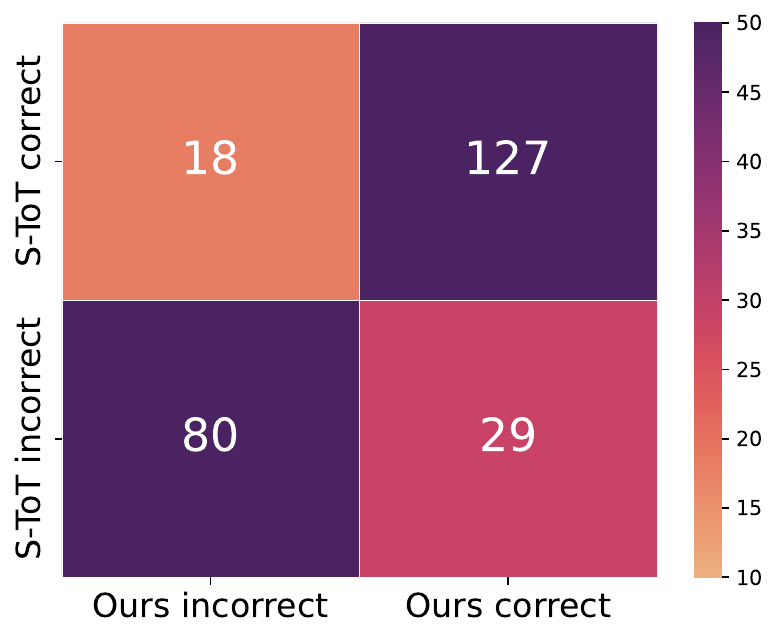}
}
\vspace{-6mm}
\caption{Predictions of SToT and Comp-SToT on AQuA.}
\label{fig:qa}
\end{minipage}
\vspace{-4mm}
\end{figure}

\textbf{Comparison Results.} We report the comparison results of our proposed approaches with other contenders in Table~\ref{tab:QA}. All CoT approaches outperform the Direct query method, showing the importance of designing effective CoTs to guide LLMs from simplicity to complexity. We can also observe that the SC-SToT method outperforms the original SToT method, where the algorithm scores the intermediate thoughts multiple times to alleviate the noise. However, this mechanism will significantly increase the token cost. We leave the detailed discussion of token cost to \textbf{Appnedix~\ref{app:experiment-cost}}. 

Both proposed variants of SToT methods achieve higher average accuracy than the original SToT method. In Comp-SToT, the point-wise scoring mechanism is replaced by a pairwise comparison, while in Back-SToT, our backtracking search algorithm replaces the original search algorithm, taking into account all previous intermediate thoughts. These results demonstrate the effectiveness of these two mechanisms, such that the proposed C-ToT (Stand.) and C-ToT (Duel.) approaches outperform all contenders. Moreover, C-ToT (Duel.) achieves superior performance by further modeling noise in the comparison.

We delve deeper to explore the benefits of the pairwise comparison mechanism to test whether it can better find the most promising intermediate thoughts. Note that we do not have access to the underlying value or order of the intermediate thoughts in each round. Therefore, we use the final prediction error as a proxy, since the depth of the tree structure in the QA datasets is shallow, limited to 1 to 3 levels. We quantify the number of QA problems correctly/incorrectly predicted by ToT and Comp-SToT and report it in Figure~\ref{fig:qa}.  Our observation shows that Comp-SToT predicts more correctly when S-ToT predicts incorrectly, showing the superiority of the pairwise comparison mechanism over the pointwise scoring mechanism. This validates the rationale of pairwise comparison mechanism.

\subsection{Game of 24}

\textbf{Task setup.}
This is a math problem where the goal is to use four numbers and basic arithmetic operations $\{+, -, *, /\}$ to get a sum of 24. For example, given the input $\{4, 9, 10, 13\}$, a viable solution would be $(10 - 4) * (13 - 9) = 24$. In our experiments, we use the same dataset as in the work of~\citet{NIPS2023:ToT} and follow their experimental setup, which consists of 1,362 problems taken from the 24-point game on 4nums.com. We have selected questions numbered 401 to 500 as our question set. Each problem consisted of four numbers selected from 1 to 13, and the goal is to formulate a calculation using these numbers to reach a total of 24. The accuracy of the solutions is scored based on whether all 4 input numbers were used and whether the result is 24.

\textbf{C-ToT Setup.}
In this task, we restore unselected intermediate thoughts in previous layers and apply pruning when the current node is inferior to previously unselected ones. We set the maximum depth of tree to 6. The computation terminates either when an answer containing the number 24 is derived, or when the maximum layer limit is reached.

We set the number of selected intermediate thoughts per layer $K=5$ and let the LLM generate a variable number of new thoughts. If the total number of new thoughts is less than or equal to twice the maximum number of selected thoughts, thoughts are moved from the ``remain list'' (a list that stores reserved thoughts) to the new node list until the number of new thoughts reaches twice the maximum number of selected thoughts or the ``remain list'' is emptied. 

We also optimize the pruning process that takes place before the comparison stage and apply it to all contenders to save tokens. Newly generated intermediate thoughts with a single number unequal to 24 is eliminated. The rest are then filtered by comparison, selecting a number of new thoughts equal to or less than the maximum number of selected thoughts. Thoughts that are not selected are added to the ``remain list''. If one of the selected thoughts contains the final answer, it is added to the ``answer list'' and the interactive process is stopped. We leave the implementation details of the thought generator and the comparison prompt to \textbf{Appendix~\ref{app:experiment-implementation-24}}.

\begin{figure}
    \centering
\begin{minipage}[b]{.34\linewidth}
\centering
\resizebox{.99\textwidth}{!}{
\begin{tabular}{lc}
\hline
Method & Accuracy \\
\hline
Direct & 8.0\% \\
CoT & 4.3\% \\
SC-CoT & 8.0\% \\
SToT & 34.3\% \\
\hline
SC-SToT & 40.0\% \\
Comp-SToT & 36.3\% \\
Back-SToT & 39.0\% \\
C-ToT (Stand.) & \textbf{40.0}\% \\
C-ToT (Duel.) & \textbf{41.0}\% \\
\hline
\end{tabular}
}
\captionof{table}{Average accuracy on Game of 24.}
\label{tab:24}
\end{minipage}%
\hspace{8mm}
\begin{minipage}[b]{.46\linewidth}
\centering
\includegraphics[clip, trim=0.17 0.22cm 0.22cm 0.12cm, width=0.99\textwidth]{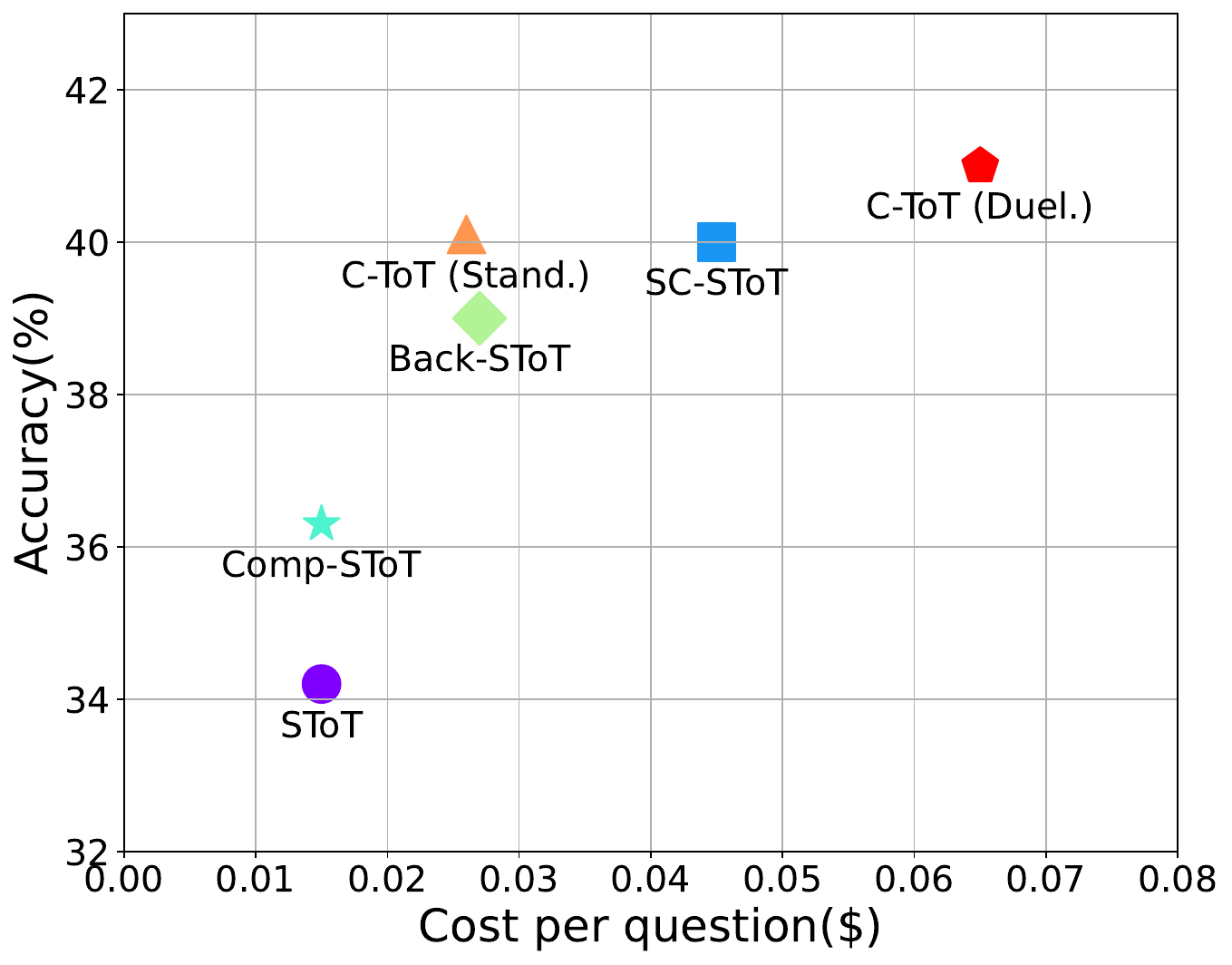}
\vspace{-6mm}
\caption{Accuracy and token costs of differenta methods.}
\label{fig:cost}
\end{minipage}
\vspace{-4mm}
\end{figure}

\textbf{Comparison Results.} 
We report the overall comparison results in Table~\ref{tab:24}. We observed trends similar to those reported in the QA task. We find that both the SToT and CToT approaches significantly outperform the CoT-based methods, indicating the need to interact with the LLM to generate a more powerful chain of thoughts to handle complex reasoning tasks. The SC-SToT method improves the performance of SToT, while the Comp-SToT can achieve a similar improvement with pairwise comparison, indicating the superiority of the comparison mechanism. Therefore, the combination of the comparison mechanism and the specific design for comparison noise allows the C-ToT (Duel.) to achieve the highest average accuracy.

We also report the average accuracy of different methods against their token cost in Figure~\ref{fig:cost}. We can observe that the token cost per question of the Comp-SToT algorithm is the same as that of the SToT algorithm, but it achieves a better performance. In practice, we can reduce the token cost by using a counter for each thought to track its comparison frequency. If the comparison count of a thought exceeds a threshold, we can remove it from the tree. We also leave the detailed discussion of token cost to \textbf{Appnedix~\ref{app:experiment-cost}}.

\subsection{Sudoku Puzzle}

\textbf{Task Setup.}
We use the Sudoku Dataset~\citep{long2023:ToT_long}, containing 10 Sudoku puzzles each of 3x3, 4x4, and 5x5 dimensions. Each puzzle is partially filled with numbers, and the task is to complete the entire Sudoku grid without changing the given numbers. The correctness of the solutions was determined by whether a complete and correct Sudoku grid is generated.

\textbf{C-ToT Setup.}
In this task, we test our proposed algorithm against other competitors in Sudoku puzzles of three different sizes. In each case, we set the maximum depth of tree of thoughts to 15. The computation stops either when the correct Sudoku solution is derived or when the maximum number of steps is reached. For each intermediate thought selected, we set $m=5$, thus generating 5 new intermediate thoughts as the next step. The maximum number of selected thoughts $K$ per layer is also set to 3. 

From the newly generated thoughts, a number equal to or less than the maximum allowed was selected by comparison. A pruning strategy was used to check for and eliminate thoughts containing results that did not meet the Sudoku requirements, such as duplicate numbers in the same row or column. These non-compliant results were removed from both the selected and unselected thoughts. The remaining unselected thoughts were then added to the ``remain list''.

If the number of selected thoughts was less than the maximum, additional thoughts were moved from the ``remain list'' to the ``select list'' until either the maximum number was reached or the ``remain list'' was emptied. We then checked whether the ``select list'' contained a correct solution. If a correct solution was found, it was added to the ``answer list'' and the program was terminated. We leave the implementation details of the thought generator and the comparison prompt to \textbf{Appendix~\ref{app:experiment-implementation-sudoku}}.

\begin{figure}
\centering
\begin{minipage}[b]{.86\linewidth}
\centering
\resizebox{.99\textwidth}{!}{
\begin{tabular}{lccc}
\hline
Method & Acc. $3\times 3$ & Acc. $4\times 4$ & Acc. $5\times 5$\\
\hline
Direct & 56.7\% & 37.7\% & 16.7\%\\
CoT & 73.3\% & 36.7\% & 23.3\%\\
SC-CoT & 76.7\% & 50.0\% & 16.7\%\\
SToT & 86.7\% & 46.7\% & 46.7\%\\
\hline
SC-SToT & 96.7\% & 53.3\% & 50.0\%\\
Comp-SToT & 100.0\% & 46.7\% & 50.0\%\\
Back-SToT & 100.0\% & 60.0\% & 56.7\%\\
C-ToT (Stand.) & \textbf{100.0}\% & \textbf{63.3}\% & \textbf{60.0}\%\\
C-ToT (Duel.) & \textbf{100.0}\% & \textbf{63.3}\% & \textbf{63.3}\%\\
\hline
\end{tabular}
}
\captionof{table}{Average accuracy on Sudoku Puzzles.}
\label{tab:sudoku}
\end{minipage}%
\vspace{-4mm}
\end{figure}

\textbf{Comparison Results.} 
We report the comparison results in Table~\ref{tab:sudoku}. In all three tasks, our proposed approaches, C-ToT (Stand.) and C-ToT (Duel.), consistently achieve the highest average accuracy, demonstrating their superior ability to handle complex reasoning tasks. While the SToT method generally outperforms the Direct method and the CoT method, it does not always outperform the SC-ToT method, as seen in the $4\times 4$ Sudoku task. This phenomenon may be due to the potential noise in pointwise scoring methods, which could mislead the subsequent generation of intermediate thoughts. The SC-ToT method introduces thought ensembles, which naturally mitigate the noise in the LLM feedback. Our proposed methods consistently outperform other contenders, suggesting that pairwise comparison mechanism effectively mitigates noise in LLM feedback, identifies the most promising intermediate thoughts, and improves the generated chain compared to SToT-based methods. We can observe that the C-ToT (Stand.) algorithm achieves the same performance as the C-ToT (Duel.) algorithm, which indicates that a single pairwise comparison could already provide reliable feedback in the Sudoku puzzle tasks.

%% file: section/appendix.tex

\section{Implementation Details}
\label{app:experiment-implementation}

In this section, we provide implementation details for all experiments, focusing primarily on the design of the thought generation and comparison prompts for each task.

\subsection{Question Answering}
\label{app:experiment-implementation-QA}
\textbf{Thought Generator.} 
We use a zero-shot prompt. For each question, we use the same prompt multiple times to generate a specified number of different new thoughts.
\begin{lstlisting}
Prompt = 'Here is a question. You should work on it step by step. Your answer must be only the alphabet of your choice and begin with ###. For example: ###A, which should be at the last line. Q: {question}'
\end{lstlisting}

\textbf{Comparison Prompt.} We use multiple different prompts to generate the comparison result at each round. For the QA problem, we use three different prompts. We evaluate the same thought three times by using each prompt once and take the majority as the answer.
\begin{lstlisting}
Prompt 1 = 'You should judge which of the two analysis is better. You must only reply 1 or 2. 
    1: {input_1} 
    2: {input_2}'
Prompt 2 = 'Find out which of the two analysis is better. You must only reply 1 or 2. 
    1: {input_1} 
    2: {input_2}'
Prompt 3 = 'Compare the two analysis and find which is better. You must only reply 1 or 2. 
    1: {input_1} 
    2: {input_2}'
\end{lstlisting}

\subsection{Game of 24}
\label{app:experiment-implementation-24}
\textbf{Thought Generator.} 
In this experiment, we use prompts similar to the ToT experiment to generate thoughts. There are two prompts: one is to select two numbers from the remaining list for the next step in the 24-point calculation, and then to add the newly obtained number back into the remaining list of numbers. The other is to generate the total operation formula that results in 24, based on all previous steps, when only one number remains. Both prompts are few-shot.
\begin{lstlisting}
Prompt 1 = 'You should choose two of the input numbers and use basic arithmetic operations (+ - * /) to obtain a new number. The new number should replace those two input numbers. Give me at least 6 possible next steps.
    Input: 2 8 8 14
    Possible next steps:
    2 + 8 = 10 (left: 8 10 14)
    8 / 2 = 4 (left: 4 8 14)
    14 + 2 = 16 (left: 8 8 16)
    2 * 8 = 16 (left: 8 14 16)
    8 - 2 = 6 (left: 6 8 14)
    14 - 8 = 6 (left: 2 6 8)
    14 /  2 = 7 (left: 7 8 8)
    14 - 2 = 12 (left: 8 8 12)
    Input: {input}
    Possible next steps:'
    
Prompt 2 = 'Use numbers and basic arithmetic operations (+ - * /) to obtain 24. Each step, you are only allowed to choose two of the remaining numbers to obtain a new number.
    Input: 4 4 6 8
    Steps:
    4 + 8 = 12 (left: 4 6 12)
    6 - 4 = 2 (left: 2 12)
    2 * 12 = 24 (left: 24)
    Answer: (6 - 4) * (4 + 8) = 24
    Input: 2 9 10 12
    Steps:
    12 * 2 = 24 (left: 9 10 24)
    10 - 9 = 1 (left: 1 24)
    24 * 1 = 24 (left: 24)
    Answer: (12 * 2) * (10 - 9) = 24
    Input: {input}'
    
\end{lstlisting}

\textbf{Comparison Prompt.} We use multiple different prompts to generate the comparison result at each round. For the 24 problem, we use three different prompts. All prompts are few-shot. We evaluate the same thought three times by using each prompt once and take the majority as the answer.
\begin{lstlisting}
Prompt 1 = 'I will give you two groups of numbers. The evaluation criteria is if using all of the given numbers with basic arithmetic operations (+ - * /) can reach 24. You should compare the two inputs and decide which input is better. You should only reply 1 or 2. 
    input_1: 2 12
    2 * 12 = 24
    input_2: 11 12
    all arithmetic operations can't get 24
    Answer: 1
    input_1: 1 2 4
    too small
    input_2: 3 8
    3 * 8 =24
    Answer: 2
    input_1: 1 12 11
    1 + 12 + 11 = 24
    input_2: 12 12
    12 + 12 = 24
    Both can reach 24, randomly select one
    Answer: 1
    input_1: {input_1}
    input_2: {input_2}
    Answer: '

Prompt 2 = 'I will give you two groups of numbers. Tell me which input is better. The better one is more possible to reach 24 by using all of the given numbers with basic arithmetic operations (+ - * /). You should only reply 1 or 2.
    //same examples
    input_1: {input_1}
    input_2: {input_2}
    Answer: '

Prompt 3 = 'Here are two groups of numbers. Tell me which input is more possible to use all of the given numbers with basic arithmetic operations (+ - * /) to get 24. You should only reply 1 or 2. Don't add any explanation.
    //same examples
    input_1: {input_1}
    input_2: {input_2}
    Answer: '
    
\end{lstlisting}

\subsection{Sudoku Puzzle}
\label{app:experiment-implementation-sudoku}
\textbf{Thought Generator.} 
We use the following prompt to generate thoughts.
\begin{lstlisting}
Prompt = 'This is a {puzzle_size}x{puzzle_size} two-dimensional array represents a matrix, where some numbers are already given, and '*' represents the numbers that need to be filled in. You should pick 1 or 2 '*' to fill in a number between 1 to {puzzle_size}. Don't change the given number. Don't complete the whole puzzle immediately until there is only 1 or 2 '*' left to be filled in. Your answer should just be the same format as the question below. When you answer, begin with ###. For example: ###[[1, *, *], [*, 1, *], [*, 2, *]]
    Question: {question}'
\end{lstlisting}

\textbf{Comparison Prompt.} We use multiple different prompts to generate the comparison result at each round. For the Sudoku problem, we use three different prompts. All prompts are zero-shot.
\begin{lstlisting}
Prompt 1 = 'You should judge which of the two two-dimensional array better represents a {puzzle_size}x{puzzle_size} Sudoku puzzle. '*' means the value is yet to be decided. You should judge by considering if in each row or column 1 to {puzzle_size} could appear and only appear once. You must only reply 1 or 2. 
    1:{input_1}
    2:{input_2}'
    
Prompt 2 = 'Find which of the two two-dimensional array better represents a {puzzle_size}x{puzzle_size} Sudoku puzzle. '*' means the value hasn't been decided. The better one should satisfy that in each row or column 1 to {puzzle_size} could appear and only appear once. You must only reply 1 or 2.
    1:{input_1}
    2:{input_2}'
    
Prompt 3 = 'Which of the two two-dimensional array better represents a {puzzle_size}x{puzzle_size} Sudoku puzzle? '*' means the value is yet to be decided. A better one means in each row or column 1 to {puzzle_size} could appear and only appear once. You must only reply 1 or 2.
    1:{input_1}
    2:{input_2}'

\end{lstlisting}

\section{Proofs}
\label{app:proofs}

We first introduce the following lemma before our main proof.
\begin{myLemma}[Lemma 2 in~\cite{ICML2017:Knockout}]
\label{lem:knockout_round}
Let $\tilde{p}(i,j)= p(i,j)-1/2$ be the additional probability by which i is preferable to j. Let $z^*$ be the maximum in $Z$ and $k^*$ be the comparison winner. The comparison algorithm on set $Z$ uses $\frac{|Z|}{4\epsilon^2}\log\frac{2}{\delta}$ comparisons and with probability $\geq 1-\delta$,
\[\tilde{p}(z^*,k^*)\leq \gamma\epsilon.\]
\end{myLemma}

\begin{proof}[Proof of Lemma~\ref{lem:knockout_round}]
To make this paper self-contained, we provide the proofs in~\cite{ICML2017:Knockout} here. First, we prove that the probability of the direct pairwise comparison process providing a wrong winner is less than $\delta$. Let $\hat{p}_i^r$ and $\hat{c}^r$ denote $\hat{p}_i$ and $\hat{c}$ respectively after $r$ number of comparisons. The output of the pairwise comparison will not be $i$ only if $\hat{p}_i^r < \frac12 + \epsilon - \hat{c}^r$ for any $r<m=\frac{1}{2\epsilon^2}\log\frac{2}{\delta}$ or if $\hat{p_i} < \frac12$ for $r = m$.

Considering the first case, after $r$ comparisons, by Chernoff bound,
\[
\Pr(\hat{p}_i^r < \frac12 + \epsilon - \hat{c}^r) \le e^{-2r(\hat{c}^r)^2} = e^{-\log \frac{4r^2}{\delta}} = \frac{\delta}{4r^2}.
\]
Using union bound,
\[
\Pr(\exists r \text{ s.t. }\hat{p}_i^r \le \frac12+\epsilon - \hat{c}^r) \le \frac{\delta}{2}.
\]
Considering the second case, after $m=\frac1{2\epsilon^2}\log\frac{2}{\delta}$ rounds, by Chernoff bound,
\begin{equation}
    \label{eqn:Chernoff}
    Pr(\hat{p}_i^m < \frac12) \le e^{-2m\epsilon^2} = \frac{\delta}{2}.
\end{equation}
Thus, the probability of each of these events happening is bounded by $\frac{\delta}{2}$, and thus the probability of the pairwise comparison process providing a wrong winner is less than $\delta$.

As each of the $|Z|/2$ pairs is compared at most $\frac{1}{2\epsilon^2}\log \frac2{\delta}$ times, the total comparisons is less than $\frac{|Z|}{4\epsilon^2}\log \frac2{\delta}$. Let $k^*$ be the comparison winner and $z^*$ be the maximum in $Z$. Let $a$ be the element paired with $z^*$. There are
two cases: $\tilde{p}(z^*,a) \ge \epsilon$ and $\tilde{p}(z^*,a) < \epsilon$.

If $\tilde{p}(z^*,a) \ge \epsilon$, by Eqn. \eqref{eqn:Chernoff} with
probability $\ge1-\delta$, $z^*$ will win and hence by definitions of
$z^*$ and $k^*$, $\tilde{p}(z^*, k^*) = 0 \le \gamma \epsilon$.
Alternatively, if $\tilde{p}(z^*,a) < \epsilon$, let $\text{winner}(i,j)$ denote
the winner between $i$ and $j$ when compared for
$\frac{1}{2\epsilon^2} \log \frac1{\delta}$ times. Then,
\[
r(a) \stackrel{(a)}\le r(\text{winner}(z^*,a)) \stackrel{(b)}\le r(k^*)
\stackrel{(c)}\le r(z^*)
\]
where (a) follows from $r(a) \le r(z^*)$, (b) and (c) follow from the
definitions of $z^*$ and $k^*$ respectively.  From strong stochastic
tranisitivity on $a$, $k^*$ and $z^*$, $\tilde{p}(z^*, k^*) \le \gamma
\tilde{p}(z^*, a) \le \gamma \epsilon$.
\end{proof}

Now we begin to prove Lemma~\ref{Lemma:Knockout}.
\begin{proof}[Proof of Lemma~\ref{Lemma:Knockout}]
To make this paper self-contained, we also provide the proofs in~\cite{ICML2017:Knockout} here. We first show that with probability $\ge 1 - \delta$, the output of Knockout is an $\epsilon$-maximum. Let $\epsilon_i = c\epsilon/2^{i/3}$ and $\delta_i = \delta/2^i$. Note that $\epsilon_i$ and $\delta_i$ are bias and confidence values used in round $i$. Let $b_i$ be a maximum element in the set $Z$ before round $i$. Then by Lemma \ref{lem:knockout_round}, with probability $\ge1-\delta_i$,
\begin{align}
\tilde{p}(b_i, b_{i+1}) \leq \frac{c \epsilon}{2^{i/3}}.
\label{eq: knockout_step}
\end{align}

By union bound, denote by $p'$ the probability that Eqn.~\eqref{eq: knockout_step} does not hold for some round $1\le i\le \log |Z|$, then we have
\[
    p' \le \sum_{i=1}^{\log |Z|} \delta_i = \sum_{i=1}^{\log |Z|} \frac{\delta}{2^i} \le \delta.
\]

With probability $\ge 1- \delta$, Eqn.~\eqref{eq: knockout_step} holds for all $i$ and by stochastic triangle inequality,
\[
\tilde{p}(b_1, b_{ \log{|Z|} + 1 }) \leq \sum_{i=1}^{{\log |Z|}}
\tilde{p}(b_i, b_{i+1}) \leq \sum^\infty_{i=1}
\frac{c\epsilon}{2^{i/3}} = \epsilon.
\]

We now bound the number of comparisons. Let $n_i = \frac{|Z|}{2^{i-1}}$ be the number of elements in the set at the beginning of round $i$. Denote by $\textnormal{NC}_i$ the number of comparisons at round $i$, then we have
\[
    \textnormal{NC}_i \le \frac{n_i}{2} \cdot \frac{\gamma^22^{2i/3}}{2c^2\epsilon^2} \cdot \log \frac{2^{i+1}}{\delta}.
\]

Hence the number of comparisons in all rounds is
\begin{align*}
\sum^{ \log |Z|}_{i=1}\frac{|Z|}{2^{i}} \cdot
\frac{\gamma^22^{2i/3}}{2c^2\epsilon^2} \cdot \log \frac{2^{i+1}}{\delta}
& \leq \frac{|Z|\gamma^2}{ 2c^2\epsilon^2} \sum^{\infty}_{i=1}
\frac{1}{2^{i/3}} \left(i+ \log \frac2{\delta}\right) \\ 
& = \frac{|Z|\gamma^2}{ 2c^2\epsilon^2} \left(\frac{2^{1/3}}{c^2} + \frac{1}{c} \log \frac{2}{\delta} \right)  = \O\left(\frac{|Z|\gamma^2}{\epsilon^2}\log\frac{1}{\delta}\right).
\end{align*}

\end{proof}

Now we begin to prove Proposition~\ref{Prop:CToT}.

\begin{proof}[Proof of Proposition~\ref{Prop:CToT}]
In each round of comparison, according to Lemma~\ref{Lemma:Knockout}, we have the probability of $1-\delta$ to output the $\epsilon$-maximum thoughts. Suppose that the thoughts in the shallower layers are more promising than those in the deeper ones, which is often the case in step-by-step reasoning tasks. For example, $r([\z^1_j]) \geq r([\z^1_i, \z^2])$ for $j\in[K]$, $i\not\in[K]$, and $\z^2\in Z^2$. When generating the intermediate thoughts in the $\tau$-th layer, the probability that the $\epsilon$-maximum thoughts are not selected is $1 - \delta^{\tau}$.

Therefore, the probability of missing the $\epsilon$-maximum promising thoughts in the $\tau$-th layer is $1-\delta^{\tau}$ with at most $\O\left(\frac{\gamma^2\sum_{i=1}^{T}|Z^i|}{\epsilon^2}\log\frac{1}{\delta}\right)$ comparisons required for generating the tree. 
\end{proof}

\section{Additional Experimental Results}
\label{app:experiment}

\subsection{Experimental Results Summary}
\label{app:experiments-summary}
First, we present a summary of the experimental results.

We introduce three additional QA datasets—Gsm8k~\citep{cobbe2021training}, Coin Flip (OOD)~\citep{NIPS2022:CoT}, and BBH~\citep{srivastava2023beyond}—along with two more state-of-the-art algorithms for comparison, with implementation details provided below.

\textbf{Contenders.} PoT~\citep{chen2023program}: PoT requires in-context samples to guide LLMs generating Python code step-by-step, and the number of samples is a hyperparameter. Since SToT and C-ToT do not always require such samples, we choose appropriate number of samples in PoT to maintain experimental consistency with other methods. We use one in-context sample for all datasets expecting Game of 24, which uses 3 samples.

Self-Refine~\citep{madaan2023self}: The parameter of in Self-Refine is the number of iterations of Self-Refine. For each task, we set the number of iterations to 3 so that the number of tokens used is approximately the same. We use the same template as in the original paper and set the number of in-context examples to the same as in the PoT.

We can observe in Table~\ref{tab:experiment-overview} that the proposed C-ToT (Stand.) and C-ToT (Duel.) algorithms achieve the best performance in almost all tasks. The Self-Refine also achieves promising results in QA tasks, while it performs relatively poorly compared to the SToT and C-ToT algorithms in complex tasks that require step-by-step interaction and reasoning, such as Game of 24.

\begin{figure}[h]
    \centering
\begin{minipage}[b]{.9\linewidth}
\centering
\resizebox{.99\textwidth}{!}{
\begin{tabular}{lcccccc}
\hline
Method / Data & Gsm8k & Coin flip (OOD) & BBH & AQuA & Game of 24 & Sudoku\\
\hline
CoT & 68.8 $\pm$ 2.5 & 56.3 $\pm$ 2.1 & 67.7 $\pm$ 2.6 & 42.3 $\pm$ 2.5 & 4.3 $\pm$ 3.2 & 44.4 $\pm$ 2.3\\
SToT & 59.3 $\pm$ 2.4 & 62.1 $\pm$ 2.9 & 69.6 $\pm$ 2.0 & 57.1 $\pm$ 2.1 & 34.3 $\pm$ 3.9 & 60.0 $\pm$ 3.8\\
PoT & 62.2 $\pm$ 3.3 & \textbf{71.1 $\pm$ 2.8} & 66.7 $\pm$ 2.3 & 47.5 $\pm$ 4.5 & 27.2 $\pm$ 3.7 & 43.1 $\pm$ 3.9\\
Self-Refine & 67.6 $\pm$ 3.2 & 64.8 $\pm$ 1.6 & 69.2 $\pm$ 3.7 & 56.2 $\pm$ 4.1 & 16.3 $\pm$ 3.9 & 52.3 $\pm$ 3.0\\
\hline
C-ToT (Stand.) & 71.3 $\pm$ 3.3 & 66.2 $\pm$ 2.8 & 75.7 $\pm$ 1.9 & 61.4 $\pm$ 2.9 & 40.0 $\pm$ 4.6 & 74.4 $\pm$ 2.1\\
C-ToT (Duel.) & \textbf{73.0 $\pm$ 2.7} & 70.4 $\pm$ 2.1 & \textbf{80.0 $\pm$ 2.2} & \textbf{63.0 $\pm$ 2.6} & \textbf{41.0 $\pm$ 3.0} & \textbf{75.4 $\pm$ 3.3}\\
\hline
\end{tabular}
}
\captionof{table}{Performance comparisons on benchmark datasets. On each dataset, 5 test runs were conducted and the average accuracy as well as standard deviation are presented, and the best one is emphasized in bold.}
\label{tab:experiment-overview}
\end{minipage}%
\end{figure}

Our method is generally not comparable to GoT~\cite{besta2023:GoT}, but the idea of comparison can improve GoT. The GoT divides a complex task into several subtasks, generate thoughts for each subtask and merge them, while the ToT-style methods aim at step-by-step thinking, where the new generated thought is based on previous ones, so it is hard to split or merge. Thus, they are better suited for different tasks. For example, for sorting, we can divide it into subtasks; while for the Game of 24, the answer should be generated based on previous thoughts, so C-ToT is more proper.

We can use the idea of comparison to improve the solving of subtasks, so we can subsequently improve the GoT. The following is a preliminary study of the sorting task in the GoT paper, remaining all its setups, and comparing the performance with scoring mechanism by LLM and the pairwise comparison mechanism by LLM in subtasks. Accuracy and token usage are reported as in Table~\ref{tab:experiment-GoT}. 5 test runs were conducted and the average accuracy and the number of tokens (completion tokens / prompt tokens) are reported.

\begin{figure}[h]
    \centering
\begin{minipage}[b]{.5\linewidth}
\centering
\resizebox{.99\textwidth}{!}{
\begin{tabular}{lcc}
\hline
Data / Method & S-GoT & C-GoT\\
\hline
32 elements & 90.4\%; 1232/12726 & 90.5\%; 1217/14969\\
64 elements & 86.1\%; 2592/31118 & 86.6\%; 2356/35313\\
\hline
\end{tabular}
}
\captionof{table}{Performance comparisons with GoT.}
\label{tab:experiment-GoT}
\end{minipage}%
\end{figure}

\subsection{Cost and Efficiency}
\label{app:experiment-cost}
In Table~\ref{tab:cost-QA},~\ref{tab:cost-24} and~\ref{tab:cost-sudoku}, we report the token cost and average accuracy comparison of our proposed approaches with other contenders. 5 test runs were conducted and the average accuracy and the number of tokens (completion tokens / prompt tokens) are reported. 

We can observe that the token costs of our proposed C-ToT approaches and the contenders are task-specific and generally incomparable. Since we preserve all previous intermediate thoughts in the QA task, the token cost of the C-ToT approaches is higher than that of the S-ToT algorithm, but the performance is simultaneously improved. In the QA and Game of 24, the token costs of the Comp-SToT and S-ToT approaches are comparable, but the Comp-SToT approaches achieve better average accuracy. In the Sudoku puzzle, the token cost of Comp-SToT is lower than that of S-ToT. These results indicate the effectiveness of using the direct pairwise comparison approach to find the most promising intermediate thoughts. In practice, we can further reduce the token cost by using a counter for each intermediate thought to track its comparison frequency. If the comparison count of an intermediate thought exceeds a threshold, we can remove it from the tree.

\begin{figure}[h]
    \centering
\begin{minipage}[b]{.6\linewidth}
\centering
\resizebox{.99\textwidth}{!}{
\begin{tabular}{lccc}
\hline
Method & Generate/Prompt tokens & Cost per case & Accuracy\\
\hline
CoT & 106/136 & 0.0003 & 42.3\% \\
SC-CoT & 1647/2023 & 0.0054 & 58.4\%\\
SToT & 1551/5415 & 0.0085 & 57.1\%\\
SC-SToT & 2081/13459 & 0.018 & 57.6\%\\
\hline
Comp-SToT & 1515/6299 & 0.0093 & 59.0\%\\
Back-SToT & 1551/8135 & 0.011 & 58.0\%\\
C-ToT (Stand.) & 1498/14627 & 0.017 & 61.4\%\\
C-ToT (Duel.) & 1649/52044 & 0.055 & 63.0\%\\
\hline
\end{tabular}
}
\captionof{table}{Average accuracy of different methods with token costs on QA.}
\label{tab:cost-QA}
\end{minipage}%
\end{figure}

\begin{figure}[h]
    \centering
\begin{minipage}[b]{.6\linewidth}
\centering
\resizebox{.99\textwidth}{!}{
\begin{tabular}{lccc}
\hline
Method & Generate/Prompt tokens & Cost per case & Accuracy\\
\hline
CoT & 99/437 & 0.0006 & 4.3\% \\
SC-CoT & 1717/6555 & 0.010 & 8.0\%\\
SToT & 1368/12205 & 0.015 & 34.3\%\\
SC-SToT & 2284/40825 & 0.045 & 40.0\%\\
\hline
Comp-SToT & 1309/11963 & 0.015 & 36.3\%\\
Back-SToT & 1679/23178 & 0.027 & 39.0\%\\
C-ToT (Stand.) & 2452/21003 & 0.026 & 40.0\%\\
C-ToT (Duel.) & 2174/60578 & 0.065 & 41.0\%\\
\hline
\end{tabular}
}
\captionof{table}{Average accuracy of different methods with token costs on Game of 24.}
\label{tab:cost-24}
\end{minipage}%
\end{figure}

\begin{figure}[h]
    \centering
\begin{minipage}[b]{.6\linewidth}
\centering
\resizebox{.99\textwidth}{!}{
\begin{tabular}{lccc}
\hline
Method & Generate/Prompt tokens & Cost per case & Accuracy\\
\hline
CoT & 431/178 & 0.001 & 44.4\% \\
SC-CoT & 6292/2666 & 0.015 & 47.8\%\\
SToT & 6309/23933 & 0.037 & 60.0\%\\
SC-SToT & 6568/70129 & 0.083 & 66.7\%\\
\hline
Comp-SToT & 2666/13164 & 0.019 & 65.6\%\\
Back-SToT & 4536/21383 & 0.030 & 72.2\%\\
C-ToT (Stand.) & 5340/29565 & 0.040 & 74.4\%\\
C-ToT (Duel.) & 7148/86425 & 0.101 & 75.5\%\\
\hline
\end{tabular}
}
\captionof{table}{Average accuracy of different methods with token costs on Sudoku.}
\label{tab:cost-sudoku}
\end{minipage}%
\end{figure}

\subsection{Ablation Studies}

We report the ablation studies on Aqua dataset, while we also observe the same trend in other datasets. 5 test runs were conducted and the average accuracy and the number of tokens (completion tokens / prompt tokens) are reported.

We study the number of intermediate thoughts generated and selected in each round and report the results in Table~\ref{tab:ablation-parameters}. With a fixed number of thoughts generated each round, selecting more thoughts leads to higher costs, while accuracy may not benefit much. With a fixed number of thoughts selected each round, generating more thoughts leads to a significant increase in accuracy because we can explore more thoughts. These results could benefit the further use of CoT methods.

\begin{figure}[h]
    \centering
\begin{minipage}[b]{.98\linewidth}
\centering
\resizebox{.99\textwidth}{!}{
\begin{tabular}{lccccc}
\hline
Generate  $m$ / Select  $K$ & $K=1$ & $K=2$ & $K=3$ & $K=5$ & $K=6$\\
\hline
$m=1$ & 41.6\% ; 50/515 & -- & -- & -- & -- \\
$m=3$ & 46.7\% ; 139/1525 & 50.2\% ; 244/2565 & 49.8\% ; 375/3658 & -- & --\\
$m=5$ & 50.1\% ; 288/2653 & 53.5\% ; 495/4452 & 55.0\% ; 549/5945 & 54.8\% ; 882/9379 & --\\
$m=10$ & 53.0\% ; 518/5189 & 56.6\% ; 896/8716 & 57.3\% ; 1091/11875 & 57.5\% ; 2120/19473 & 57.6\% ; 2155/20982\\
$m=12$ & 53.8\% ; 674/6334 & 57.4\% ; 1212/10732 & 61.4\% ; 1498/14627 & 61.4\% ; 1531/16002 & 61.5\% ; 2125/26329\\
\hline
\end{tabular}
}
\captionof{table}{Ablation studies on parameters.}
\label{tab:ablation-parameters}
\end{minipage}%
\end{figure}

We study the threshold for removing intermediate thoughts vs. accuracy and cost and report the results in Table~\ref{tab:ablation-thresholds}. Under the same tree depth, different thresholds show relatively stable performance.

\begin{figure}[h]
    \centering
\begin{minipage}[b]{.98\linewidth}
\centering
\resizebox{.99\textwidth}{!}{
\begin{tabular}{lccccc}
\hline
Depth  $d$ / Threshold Th & Th$=1$ & Th$=2$ & Th$=3$ & Th$=4$ & Th$=5$\\
\hline
$d=3$ & 58.0\% ; 1062/11337 & 60.7\% ;  1315/13611 & 61.0\% ; 1478/15111 & 61.4\% ; 1513/15713 & 61.4\% ; 1558/16107\\
$d=4$ & 58.8\% ; 1426/13792 & 60.9\% ; 1732/18822 & 61.4\% ; 1771/19940 & 61.6\% ; 1804/20003 & 61.3\% ; 2001/20312\\
$d=5$ & 60.0\% ; 1760/17556 & 61.0\% ; 2143/21006 & 61.6\% ; 2271/23412 & 61.7\% ; 2265/24031 & 61.3\% ; 2425/24755\\
\hline
\end{tabular}
}
\captionof{table}{Ablation studies on the number of thresholds for removing intermediate thoughts.}
\label{tab:ablation-thresholds}
\end{minipage}%
\end{figure}

We study the number of comparison and report the results in Table~\ref{tab:ablation-comparisons}. In general, increasing the number of comparison results in better accuracy and higher cost.

\begin{figure}[!h]
    \centering
\begin{minipage}[b]{.9\linewidth}
\centering
\resizebox{.99\textwidth}{!}{
\begin{tabular}{lcccc}
\hline
Dataset / Comparison $n$ & $n=1$ & $n=3$ & $n=5$ & $n=8$\\
\hline
Aqua & 61.4\% ; 1498/14627 & 63.0\% ;  1649/52044 & 64.3\% ; 1824/78194 & 64.7\% ; 1997/128802 \\
\hline
\end{tabular}
}
\captionof{table}{Ablation studies on the number of comparisons.}
\label{tab:ablation-comparisons}
\end{minipage}%
\end{figure}